\documentclass[twoside]{article}

\usepackage{microtype}
\usepackage{graphicx}
\usepackage{subcaption}
\usepackage{booktabs} 
\usepackage{amsfonts}
\usepackage{algorithm}
\usepackage{url}
\usepackage{algpseudocode}
\usepackage[table]{xcolor}
\usepackage{multirow}
\usepackage{csquotes}
\usepackage{bm}
\usepackage{enumitem}
\usepackage{tikz}
\usepackage{dblfloatfix}

\usepackage[round]{natbib}

\usepackage{hyperref}

\usepackage[accepted]{aistats2026}

\usepackage{amsmath}
\usepackage{amssymb}
\usepackage{mathtools}
\usepackage{amsthm}

\newcommand{\mainrel}[2]{%
  #1 {\scriptsize\textcolor{black!80}{(+#2\%)}}%
}

\newcommand{\avgrel}[2]{%
  \textcolor{black!100}{#1}%
}
\newcommand{\avgrelbold}[1]{%
  \textbf{\textcolor{black!100}{#1}}%
}

\usepackage[capitalize,noabbrev]{cleveref}

\theoremstyle{plain}

\newtheorem{proposition}{Proposition}[section]

\theoremstyle{definition}

\theoremstyle{remark}

\newcommand{\params}{\boldsymbol{\theta}}

\newcommand{\weightdecay}{\lambda}

\usepackage[textsize=tiny]{todonotes}

\begin{document}

\twocolumn[

\aistatstitle{Guiding Posterior Exploration with Optimizer-Derived Geometry}

\aistatsauthor{ Moritz Schlager$^{*}$ \And Emanuel Sommer$^{*}$ \And Thomas M\"ollenhoff \And David R\"ugamer }

\aistatsaddress{ TUM, MCML \And  LMU Munich, MCML \And RIKEN AIP \And LMU Munich, MCML }]

\begin{abstract}
  Sampling-based methods offer a principled approach to uncertainty quantification in Bayesian neural networks. Their practical use, however, is often challenged by the computational cost of exploring high-dimensional and multimodal posterior distributions. To overcome these difficulties, Bayesian Deep Ensembles, i.e., warmstarting the sampling from several optimized solutions, have proven to be an effective strategy. In this paper, we demonstrate that curvature estimates computed during the warmstart as a byproduct in adaptive optimizers such as AdamW can inform the sampling phase at negligible additional cost. 
  Specifically, our proposed preconditioned sampling strategy based on optimizer-derived geometries can substantially reduce or even eliminate the need for a lengthy sampling burn-in phase and leads to greater numerical stability. This approach consistently maintains or improves predictive performance and uncertainty quantification without any additional computational costs. We confirm the consistency of our findings across various datasets and network architectures.
\end{abstract}

\section{INTRODUCTION}

Bayesian deep learning strives to build a principled framework for uncertainty quantification, a critical requirement for deploying neural networks in high-stakes domains. Sampling-based methods such as Markov chain Monte Carlo (MCMC) offer a powerful means of approximating the posterior distribution over the model's parameters. However, their practical application is often hindered by the immense computational cost of exploring the complex, high-dimensional, and generally ill-conditioned posterior landscapes typical of deep Bayesian neural networks (BNNs) \citep{papamarkou_position_2024}.

\begin{figure}[t]
    \centering
    \includegraphics[width=1\linewidth]{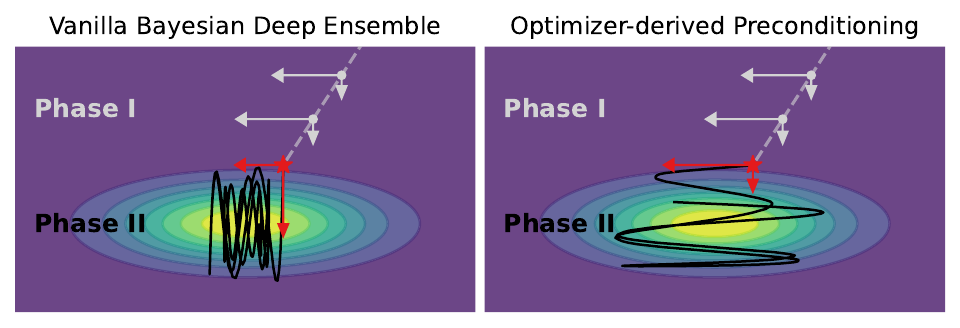}
    \caption{Simulation of optimizer-derived preconditioning in a two-phase Bayesian Deep Ensemble pipeline on an ill-conditioned Gaussian mixture. Phase I shows AdamW optimization trajectories and associated coordinate-wise second-moment statistics (grey arrows). Phase II depicts SGHMC sampling trajectories (black line) and the used preconditioning magnitudes (red arrows). Left: Vanilla sampling with an identity mass matrix explores inefficiently due to strong anisotropy. Right: Optimizer-derived preconditioning aligns the sampler with the local geometry, enabling more efficient exploration and faster mixing.}
    \label{fig:precon-example-motivation}
\end{figure}

A recent and growingly popular post-Bayesian approach for scalable inference in BNNs is the two-stage BDE procedure \citep{sommer_connecting_2024, sommer2026can, paulin_sampling_2025, duffield2025scalable}. This approach first uses standard optimization to find multiple distinct modes akin to Deep Ensembles \citep{lakshminarayanan2017simple} and then initializes independent MCMC samplers from these solutions to explore their local basins of attraction. This strategy has shown strong empirical performance, often outperforming common variational methods \citep{kobialka2026interplay}. However, the optimization and sampling stages are typically treated as disjoint processes. In such cases, the geometric information about the loss landscape gathered during optimization is discarded, forcing the subsequent sampling phase to rediscover the local curvature from scratch during a potentially long and costly scale-adaptation period.

\paragraph{Our Contributions}

In this work, we explore a tighter integration of the optimization and sampling phases. We use information of adaptive optimizers like AdamW \citep{loshchilov_decoupled_2019}, which maintains first- and second-moment estimates derived from gradients as an inexpensive proxy for local curvature. These optimizer-derived statistics can be utilized to directly initialize the preconditioner for subsequent stochastic gradient MCMC (SGMCMC) sampling \citep{girolami_riemann_2011}  (see \cref{fig:precon-example-motivation}). By providing an informed geometric starting point, this ``warmstart'' allows the sampler to bypass the potentially lengthy scale-adaptation period \citep{springenberg_bayesian_2016}, offering a more efficient path to posterior exploration in Bayesian ensembles at a negligible cost. To this end, we \vspace{-0.2cm}
\begin{itemize}[leftmargin=1.2em]
    \setlength\itemsep{0em}
    \item derive several theoretically grounded ways of using optimizer properties to precondition samplers;
    \item evaluate multiple optimizer-derived geometry transfer strategies across modalities and architectures;
    \item assess robustness and the ability to accelerate sampler convergence while measuring the impact on predictive performance and uncertainty quality;
    \item analyze the local geometry patterns found in realistic BNNs and uncover differences between static and adaptive preconditioning variants.
\end{itemize}

\section{BACKGROUND}
\label{sec:background}

In Bayesian deep learning, we aim to infer the posterior distribution $p(\params|\mathcal{D})$ over a neural network's parameters. As exact inference is intractable, we rely on approximations such as Stochastic Gradient MCMC (SGMCMC). A recent efficient paradigm is the two-stage \emph{Bayesian Deep Ensemble} (BDE) approach \citep{sommer_connecting_2024, duffield2025scalable, paulin_sampling_2025}, which initializes independent MCMC chains from high-probability modes found by optimization.

A major challenge in this domain is the efficient exploration of the complex posterior geometry, which often requires preconditioning, for instance through a mass matrix $\mathbf{M}$ in SGHMC. While effective, such preconditioning schemes typically rely on expensive adaptation phases, as in scale-adapted SGHMC \citep[AdaSGHMC;][]{springenberg_bayesian_2016}. Our work focuses on bridging the optimization and sampling stages by utilizing the curvature estimates (such as those from AdamW) generated during the first stage as a "free," high-quality preconditioner for the sampling phase. 

We provide a comprehensive background on Bayesian inference, BDEs, detailed SGMCMC dynamics, and a full review of related literature in App. \ref{app:related}.

\begin{table*}[!b]
    \centering
    \small
    \caption{Average Performance (3 replications) of layer-wise preconditioned SGHMC (no warmup) and AdaSGHMC (with warmup) across all datasets. Relative improvements over the Deep Ensemble baseline are shown in parentheses. Best values per column are boldfaced. Further details and results in Appendices \ref{app:furtherlayerwiseperf} and \ref{app:furtherparamwiseperf}.}
    \setlength{\tabcolsep}{4pt}
    \begin{subtable}{\textwidth}
        \centering
        \caption{Predictive performance is reported as root mean squared error (RMSE $\downarrow$) on regression tasks and Accuracy ($\uparrow$) on classification tasks.}
        \label{tab:results_predperf_abs_rel}
        \resizebox{\textwidth}{!}{
        \begin{tabular}{ll|
            c c
            c c c
            |c}
        \toprule
        \multicolumn{1}{c}{\multirow{2}{*}{\textbf{Sampler}}} &
        \multicolumn{1}{c}{\multirow{2}{*}{\textbf{Preconditioning}}} &
        \multicolumn{2}{c}{\textbf{RMSE} ($\downarrow$)} &
        \multicolumn{3}{c}{\textbf{Accuracy} ($\uparrow$)} &
        \multicolumn{1}{c}{\multirow{2}{*}{\textbf{Avg. Improv.}}} \\
        \cmidrule(lr){3-4}\cmidrule(lr){5-7}
          & \multicolumn{1}{c@{}}{} &
        \multicolumn{1}{c}{\textbf{bike}} &
        \multicolumn{1}{c}{\textbf{protein}} &
        \multicolumn{1}{c}{\textbf{CIFAR-10}} &
        \multicolumn{1}{c}{\textbf{F-MNIST}} &
        \multicolumn{1}{c}{\textbf{shakespeare}} & \textbf{Over DE} \\
        \midrule
        Deep Ensemble & -- &
        0.2445 & 0.6589 & 0.9019 & 0.9157 & 0.5504 & -- \\
        \midrule
        SGHMC & identity &
        \mainrel{\textbf{0.2392}}{2.20} & \mainrel{0.6533}{0.86} &
        \mainrel{0.9150}{1.46} & \mainrel{0.9226}{0.75} & \mainrel{0.5536}{0.58} &
        \avgrel{+1.17\%}{} \\
        SGHMC & gradient &
        \mainrel{0.2406}{1.59} & \mainrel{0.6492}{1.47} &
        \mainrel{\textbf{0.9155}}{1.51} & \mainrel{\textbf{0.9268}}{1.21} & \mainrel{0.5533}{0.53} &
        \avgrel{+1.26\%}{} \\
        SGHMC & 1st moment &
        \mainrel{0.2405}{1.64} & \mainrel{0.6504}{1.30} &
        \mainrel{0.9141}{1.35} & \mainrel{0.9256}{1.08} & \mainrel{\textbf{0.5543}}{0.71} &
        \avgrel{+1.22\%}{} \\
        SGHMC & 2nd moment &
        \mainrel{0.2400}{1.87} & \mainrel{\textbf{0.6470}}{1.81} &
        \mainrel{0.9154}{1.50} & \mainrel{0.9231}{0.81} & \mainrel{0.5535}{0.56} &
        \avgrelbold{+1.31\%} \\
        \midrule
        AdaSGHMC & adaptive &
        \mainrel{0.2424}{0.87} & \mainrel{0.6560}{0.44} &
        \mainrel{0.9055}{0.39} & \mainrel{0.9197}{0.43} & \mainrel{0.5523}{0.35} &
        \avgrel{+0.50\%}{} \\
        \bottomrule
        \end{tabular}
        }
    \end{subtable}
    
    \vspace{.8em}
    
    \begin{subtable}{\textwidth}
        \caption{Uncertainty quantification is reported as log pointwise predictive density (LPPD $\uparrow$) for regression and classification tasks and perplexity ($\downarrow$) for language modeling tasks.}
        \centering
        \label{tab:results_uncertainty_abs_rel}
        \resizebox{\textwidth}{!}{
        \begin{tabular}{ll|
            c c c c c
            |c}
        \toprule
        \multicolumn{1}{c}{\multirow{2}{*}{\textbf{Sampler}}} &
        \multicolumn{1}{c}{\multirow{2}{*}{\textbf{Preconditioning}}} &
        \multicolumn{4}{c}{\textbf{LPPD} ($\uparrow$)} &
        \multicolumn{1}{c}{\textbf{Perplexity} ($\downarrow$)} &
        \multicolumn{1}{c}{\multirow{2}{*}{\textbf{Avg. Improv.}}} \\
        \cmidrule(lr){3-6}\cmidrule(lr){7-7}
          & \multicolumn{1}{c@{}}{} &
        \multicolumn{1}{c}{\textbf{bike}} &
        \multicolumn{1}{c}{\textbf{protein}} &
        \multicolumn{1}{c}{\textbf{CIFAR-10}} &
        \multicolumn{1}{c}{\textbf{F-MNIST}} &
        \multicolumn{1}{c}{\textbf{shakespeare}} & \textbf{Over DE}\\
        \midrule
        Deep Ensemble & -- &
        0.5834 & -0.7173 & -0.3047 & -0.2318 & 4.3687 & -- \\
        \midrule
        SGHMC & identity &
        \mainrel{0.6190}{6.11} & \mainrel{-0.6434}{10.30} &
        \mainrel{-0.2641}{13.3} & \mainrel{-0.2216}{4.40} &
        \mainrel{\textbf{4.2645}}{2.38} &
        \avgrel{+7.30\%}{} \\
        SGHMC & gradient &
        \mainrel{\textbf{0.6418}}{10.0} & \mainrel{-0.6171}{13.97} &
        \mainrel{\textbf{-0.2583}}{15.3} & \mainrel{\textbf{-0.2102}}{9.34} &
        \mainrel{4.2846}{1.92} &
        \avgrelbold{+10.1\%} \\
        SGHMC & 1st moment &
        \mainrel{0.6416}{9.99} & \mainrel{-0.6274}{12.54} &
        \mainrel{-0.2644}{13.3} & \mainrel{-0.2118}{8.64} &
        \mainrel{4.2801}{2.03} &
        \avgrel{+9.29\%}{} \\
        SGHMC & 2nd moment &
        \mainrel{0.6397}{9.67} & \mainrel{\textbf{-0.6002}}{16.32} &
        \mainrel{-0.2615}{14.2} & \mainrel{-0.2214}{4.48} &
        \mainrel{4.2894}{1.81} &
        \avgrel{+9.29\%}{} \\
        \midrule
        AdaSGHMC & adaptive &
        \mainrel{0.6182}{5.98} & \mainrel{-0.6908}{3.69} &
        \mainrel{-0.2942}{3.45} & \mainrel{-0.2197}{5.25} &
        \mainrel{4.3382}{0.70} &
        \avgrel{+3.81\%}{} \\
        \bottomrule
        \end{tabular}
        }
    \end{subtable}
    \label{tab:results_combined}
\end{table*}

\section{BRIDGING THE GAP BETWEEN OPTIMIZATION AND SAMPLING}
\label{sec:methods}

Optimization and sampling both rely on stochastic gradients and auxiliary geometry statistics. While structurally similar, the transfer of terminal optimizer statistics to precondition samplers remains unexplored. We formalize and evaluate strategies to transfer these estimates---ranging from local gradients to smoothed curvature proxies---for efficient posterior exploration.

\subsection{Parallels Between Optimizers and Samplers}

Let $\params_t \in \mathbb{R}^d$ denote the parameters at iteration $t$, and $\mathbf g_t := \nabla_\theta \hat{\ell}(\params_t)$ be the stochastic gradient of the negative log-likelihood loss $\hat{\ell}$ evaluated on a mini-batch at time $t$. Adaptive optimizers such as AdamW maintain exponential moving averages of these gradients:
\begin{align}
\boldsymbol\mu_{t} &= \beta_1 \boldsymbol\mu_{t-1} + (1-\beta_1) \mathbf g_{t-1}, \\
\boldsymbol\nu_{t} &= \beta_2 \boldsymbol\nu_{t-1} + (1-\beta_2) \mathbf g_{t-1}^2,
\end{align}
where $\boldsymbol\mu_t$ acts as momentum and $\boldsymbol\nu_t$ as a curvature proxy (gradient variance), with decay rates $\beta_{1,2} \in (0,1)$.
Similarly, SGMCMC uses a preconditioner $\mathbf M$ to rescale updates (see App.~\ref{app:related-sgmcmc}). Notably, AdaSGHMC adapts $\mathbf M$ using statistics structurally identical to $\boldsymbol\nu_t$, utilizing the posterior gradient $\nabla \tilde{U}(\params)$ rather than $\mathbf g_t$.

\paragraph{Local Spectral Alignment}
In the following, we elaborate on the theoretical bridge between Phase I and II in BDEs, allowing us to reuse information obtained during optimization for the subsequent sampling. In general, a theoretical link between the optimization loss and the posterior landscape follows from the correspondence between regularization and prior choice. For $L_2$-regularization with strength $\lambda > 0$, a mathematical equivalence between the regularized negative log-likelihood and the unnormalized log-posterior is obtained by using a Gaussian prior $\params_t \sim \mathcal{N}(\mathbf{0},\rho^2 \mathbf{I})$ with $\lambda = (2\rho^2)^{-1}$, yielding $\nabla \tilde U(\params_t) = \mathbf g_t + \lambda \params_t$. Beyond this static alignment, we justify the transfer of geometric information through the following proposition, whose proof is deferred to App.~\ref{app:proof-3.1}.

\begin{proposition}
\label{prop:isotropic-sampling}
Assume the posterior potential $U(\params)$ is locally quadratic in a basin $\mathcal{B}(\params^*)$ around a mode $\params^*$, such that $U(\params) \approx \frac{1}{2}(\params - \params^*)^\top \mathbf{H}^* (\params - \params^*)$ where $\mathbf{H}^* = \nabla^2 U(\params^*)$ is the Hessian at the mode. If the mass matrix is set to $\mathbf{M} = \mathbf{H}^*$, the preconditioned SGHMC dynamics achieve optimal local mixing by inducing an isotropic Gaussian approximation of the target distribution with a unit condition number.
\end{proposition}

Notably, by employing a static mass matrix $\mathbf{M}$ derived at the mode, the dynamics satisfy the constant fluctuation-dissipation theorem. This allows the system to exploit the local geometry while preserving the correct stationary distribution, without requiring the computationally expensive time-dependent friction terms or Christoffel-symbol corrections typically necessitated by Riemannian Manifold MCMC \citep{girolami_riemann_2011}.
Further, in high-dimensional neural networks, the full Hessian $\mathbf{H} = \nabla^2 U(\params)$ is computationally intractable and sampling is usually done using a diagonal mass matrix $\mathbf{M}$. 
In this setting, the optimizer's terminal state provides various computationally efficient diagonal proxies for $\mathbf{H}^*$ such as $\boldsymbol{\nu}_T$. For models where the stochastic gradient noise is approximately Gaussian, the second-moment accumulator $\boldsymbol{\nu}_T$ serves as a diagonal estimator of the Fisher Information Matrix (FIM), which coincides with the expected Hessian \citep{khan2018fast}. 
In practice, AdamW-style optimizers utilize root-scaling for numerical stability in non-convex landscapes where curvature may be singular or noisy. Thus, initializing $\mathbf{M} = \mathrm{diag}(\sqrt{\boldsymbol{\nu}_T} + \epsilon \mathbf{I})$ transfers the specific ``effective'' metric tensor that successfully navigated the optimization landscape to the sampling phase. 
This construction serves as a regularized spectral surrogate, preserving local geometry while ensuring the dynamics remain well-posed.

\subsection{Optimizer-Informed Preconditioning Strategies}

We consider diagonal mass matrices $\mathbf{M} = \mathrm{diag}(m_1,\dots,m_d)$ that reuse optimizer statistics accumulated during training. Beyond the isotropic baseline $\mathbf{M}=\mathbf{I}_d$, we study the following constructions:
\begin{align*}
    \text{Gradient-based:} \quad & \mathbf{M} = \mathrm{diag}(\mathbf g_T^2 + \epsilon\bm{1}), \\
    \text{First-moment:} \quad & \mathbf{M} = \mathrm{diag}(\boldsymbol\mu_T^2 + \epsilon\bm{1}), \\
    \text{Second-moment:} \quad & \mathbf{M} = \mathrm{diag}(\sqrt{\boldsymbol\nu_T} + \epsilon\bm{1}). 
\end{align*}
Here $ \mathbf g_T, \boldsymbol\mu_T, \boldsymbol\nu_T$ denote the final gradient, first-moment and second-moment statistics, respectively. 
This unified formulation highlights that optimizer states readily provide curvature proxies for initializing $\mathbf{M}$.

These strategies differ fundamentally in how they approximate posterior geometry. A preconditioner based on $\mathbf g_T$ is purely local, reflecting the stochastic gradient at the final iteration. While this can locally yield highly adaptive scaling, it is also susceptible to noise and batch-to-batch variability, and thus provides only a myopic proxy for steepness. Using $\boldsymbol\mu_T$ introduces temporal smoothing: since $\beta_1<\beta_2$ in practice, the first-moment average emphasizes more recent gradients, yielding a preconditioner that captures local structure but with improved stability. In contrast, $\boldsymbol\nu_T$ accumulates squared gradients with usually slower forgetting, effectively approximating parameter-wise Fisher information. This makes the preconditioner less responsive to local fluctuations but more robust and theoretically aligned with variance-based geometry. Following the RMSprop-inspired approach from \citet{li_preconditioned_2016}, we scale $\boldsymbol\nu_T$ by taking the element-wise square root. 

\subsection{Layer-wise Preconditioning}
To improve numerical stability, we also evaluate a \emph{layer-wise} strategy \citep{wenzel_how_2020}. Instead of maintaining a value for every parameter, we compute a single scalar for each layer $\ell$ by averaging the optimizer statistics of its parameters and normalizing across layers; the precise formulation is given in App.~\ref{app:methods_details}. This reduces granularity but acts as a regularizer against extreme single-parameter values and upper bounds the maximum parameter-wise step size, increasing stability.

\section{EMPIRICAL STUDY}

We evaluate optimizer-derived preconditioning on tabular regression (bikesharing, protein), image classification (F-MNIST, CIFAR-10), and language modeling (shakespeare) using FCNs, CNNs, ResNets, and Transformers. We compare our preconditioned SGHMC (no warmup) against Deep Ensembles and AdaSGHMC (with warmup) (see App.~\ref{app:experimental-setup} for the full setup).

\subsection{Performance and Efficiency}
\label{sec:performance}
As shown in \cref{tab:results_combined}, optimizer-derived preconditioning consistently outperforms the Deep Ensemble baseline across all tasks and metrics. Notably, it also surpasses AdaSGHMC, particularly in uncertainty quantification (LPPD, Perplexity), despite using zero sampling warmup. When compared to Vanilla SGHMC, an optimistic baseline that, due to higher hyperparameter sensitivity, benefits most from reporting peak performance over the same uniform step-size sweep applied to all methods, optimizer-derived preconditioning performs on par or better across all datasets. While identity preconditioning already improves over DEs, incorporating optimizer-derived geometry yields additional and consistent gains, most notably for UQ on tabular and image data.
We also analyzed the impact of warmup duration (see Figure~\ref{fig:warmup-effects-fmnist}; App.~\ref{app:warmup_analysis}). Local strategies (Gradient, 1st-Moment) achieve peak performance immediately (with zero warmup), making them ideal for budget-constrained settings. The 2nd-Moment strategy benefits from a longer warmup, eventually yielding the highest overall performance by leveraging global curvature information. This allows for a significant reduction in computational budget while maintaining or exceeding baseline performance. A dedicated discussion on runtime savings is provided in the App.~\ref{app:runtime}.

\subsection{Robustness and Stability}

\begin{figure}[t]
    \centering
    \includegraphics[width=0.9\linewidth]{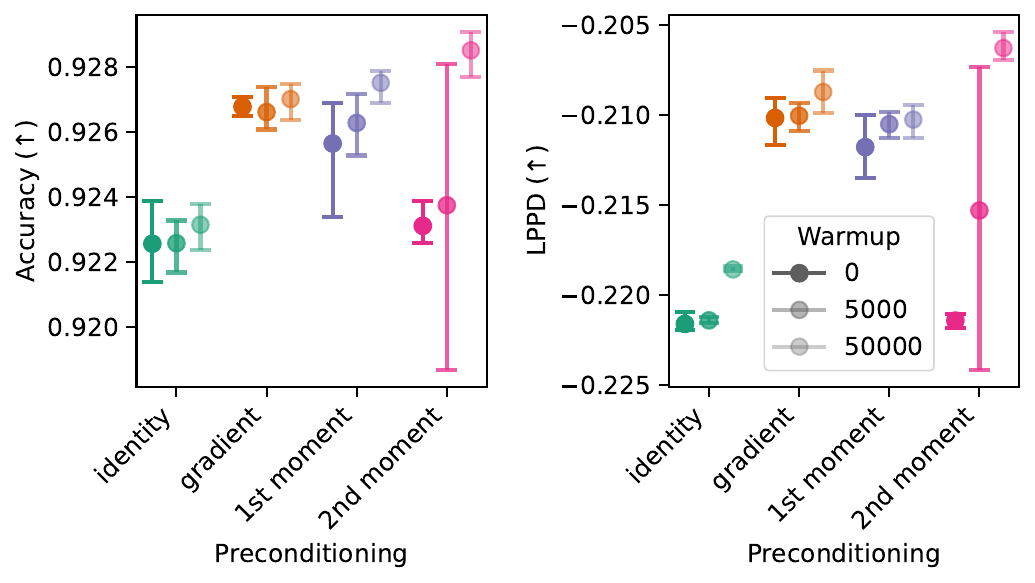}
    \caption{Minimal, maximal, and average performance on the F-MNIST dataset for different warmup lengths after warmstarting with AdamW and sampling with (preconditioned) SGHMC.}
    \label{fig:warmup-effects-fmnist}
    \vspace{-1em} 
\end{figure}

Optimizer-informed preconditioning significantly enhances robustness to the step size $\eta$. While vanilla SGHMC degrades sharply with suboptimal $\eta$, our preconditioned variants maintain stable performance across a broad range. The layer-wise normalization (\cref{sec:methods}) acts as a safeguard, rescaling effective updates to prevent discretization error (see App.~\ref{app:rob}). We further compared parameter-wise vs. layer-wise granularity and found the layer-wise approach to be substantially more numerically stable without compromising predictive performance (see App.~\ref{app:furtherparamwiseperf}).

\subsection{Analysis of Adapted Geometry}
We investigated why the static transfer outperforms the adaptive AdaSGHMC. By analyzing the cosine similarity of mass matrices (see App.~\ref{app:adaptive_analysis}), we found that AdaSGHMC's internal adaptation tends to ``overwrite" the initial geometry, making the chains in the ensemble highly self-similar. In contrast, our static optimizer-derived preconditioning preserves the distinct geometric information of each mode found in Phase I. This preservation of ensemble diversity likely accounts for the superior generalization performance. An analysis of empirical preconditioning distributions that additionally supports the layer-wise approach is provided in App.~\ref{app:empprecond}.

\section{DISCUSSION}

Our study shows that transferring optimizer-derived geometry to an SGMCMC sampler is an effective, zero-cost strategy for improving BDEs. This approach consistently retains and often enhances predictive performance and robustness, while reducing computational cost by eliminating the warmup phase. 
Our results suggest using gradient or first-moment preconditioning under constrained budgets, while second-moment estimates yield superior results given longer runs.
Our ablations confirm that layer-wise preconditioning is more numerically stable than parameter-wise counterparts. We also identified that adaptive samplers overwrite the transferred geometry, reducing ensemble diversity and providing evidence for their inferior performance.

\bibliography{refs}
\bibliographystyle{apalike}

\newpage
\appendix
\onecolumn
\section{EXTENDED BACKGROUND \& RELATED WORK}
\label{app:related}

\subsection{Bayesian Inference for Neural Networks}
In Bayesian deep learning, we aim to infer the posterior distribution over a neural network's parameters $\params \in \mathbb{R}^d$ given a dataset $\mathcal{D} = \{(y_i,\mathbf x_i)\}_{i=1}^N$. According to Bayes' rule, the posterior is $p(\params|\mathcal{D}) \propto p(\mathcal{D}|\params)p(\params)$, where $p(\mathcal{D}|\params)$ is the likelihood and $p(\params)$ is the prior. Throughout, we assume isotropic Gaussian priors on $\params$ \citep{kobialka2026interplay}. As this posterior is intractable for modern network architectures, we rely on approximation methods. Sampling-based approaches, such as SGMCMC, are a powerful class of methods that can asymptotically recover the true posterior without the restrictive assumptions of variational inference \citep{izmailov_what_2021}.

A major challenge for SGMCMC is efficiently exploring the high-dimensional, multi-modal posteriors of modern neural networks. A recently proposed remedy is a two-stage approach \citep{sommer_connecting_2024, sommer2026can, duffield2025scalable, paulin_sampling_2025}, also referred to as \emph{Bayesian Deep Ensemble}. It first uses conventional optimization to find multiple distinct, high-probability regions (modes) in the parameter space. Subsequently, independent MCMC chains are initialized from these modes to explore the local posterior geometry. Prior work has investigated optimal BDE chain initialization under constrained computational budgets \citep{rundel2025efficiently}. Our work focuses on creating a more synergistic link between these optimization and sampling stages, which, so far, have mostly been treated as independent components in the literature.

\subsection{Related Methods} 
In recent years, various approaches have been proposed to connect optimization and sampling and use synergies between both, including examples such as tempering schedules \citep{duffield2025scalable} or repeated cycles that alternate between optimization-based exploration and sampled exploitation \citep{zhang_cyclical_2020}. In this paper, we focus on the discrete two-stage approach implemented by BDEs due to its modularity and wide applicability, including, for example, the use of pre-trained models or publicly available optimization recipes. A different yet related research question in the field of sampling-based inference for BNNs is whether geometric adaptation of sampler dynamics can boost their efficiency. Despite foundational work motivating the use of preconditioning to navigate complex posteriors \citep{girolami_riemann_2011, ma_complete_2015, wenzel_how_2020}, its adoption in practice has been notably sparse. As such, prominent recent work often forgoes preconditioning entirely \citep[see e.g.][]{sommer_microcanonical_2025}, and popular libraries like \texttt{posteriors} \citep{duffield2025scalable} currently lack support for preconditioned sampling. This reluctance may stem from the practical challenges of determining and adapting the preconditioner, which can introduce its own set of hyperparameters or might require an extended warmup phase for scale-adaption \citep{springenberg_bayesian_2016}. 

\subsection{Preconditioned Stochastic Gradient MCMC} \label{app:related-sgmcmc}
Stochastic gradient Hamiltonian Monte Carlo (SGHMC) is often considered the state-of-the-art SGMCMC algorithm. It introduces an auxiliary momentum variable $\mathbf{r}$ to encourage more efficient exploration of the target distribution \citep{chen_stochastic_2014}. Its dynamics are governed by a mass matrix $\mathbf{M}$, which can be viewed as a \emph{preconditioner} that adapts the sampler's steps to the local geometry of the posterior. The discretized update equations are
\begin{align}
    \!\params_{t+1} &\!=\! \params_t + \eta \mathbf{M}^{-1} \mathbf{r}_t, \\ \label{eq:positionupdate}
    \!\mathbf{r}_{t+1} &\!=\! \mathbf{r}_t \!-\! \eta\nabla \tilde{U}(\params_t) \!-\! \eta \mathbf C \mathbf{M}^{-1} \mathbf{r}_t \!+\! \sqrt{2\eta(\mathbf{C} \!-\! \hat{\mathbf E})} \mathbf Z_t,
\end{align}
where $\eta$ is the step size, $\mathbf C$ is a friction matrix, $\nabla \tilde{U}(\params) = -\nabla \log p(\params) -\frac{N}{|\mathcal{B}|}\sum_{(\mathbf x,y) \in \mathcal{B}} \nabla \log p(y|\mathbf{x},\params)$ is the stochastic gradient of the unnormalized log-posterior based on mini-batch $\mathcal{B}\subset \mathcal{D}$,
$\hat{\mathbf E}$ is an estimate of the gradient noise covariance, and $\mathbf Z_t \sim \mathcal{N}(\bm 0, \mathbf I)$.

The choice of $\mathbf{M}$ can have a considerable effect on the sampling effectiveness. A sampler using an isotropic mass matrix ($\mathbf{M} \propto \mathbf{I}$) potentially struggles in ill-conditioned landscapes where curvature varies drastically across dimensions \citep{girolami_riemann_2011}, leading to inefficient exploration (see Figure~\ref{fig:precon-example-motivation}, left). A well-chosen preconditioner reshapes the geometry, enabling larger, more effective steps (see Figure~\ref{fig:precon-example-motivation}, right). To automate this, methods like scale-adapted SGHMC \citep[AdaSGHMC;][]{springenberg_bayesian_2016} dynamically adapt a diagonal mass matrix $\mathbf{M} = \text{diag}(\hat{\mathbf v}^{1/2})$ during a warmup phase, where $\hat{\mathbf v}$ is a running estimate of the squared gradient of $\tilde{U}(\params)$. However, this adaptation period can be lengthy and adds computational overhead.

\subsection{Curvature Estimation in Adaptive Optimizers}
Modern adaptive optimizers are designed to handle the ill-conditioned loss landscapes of deep networks. They do so by maintaining running estimates of local curvature, which are readily available as a free byproduct upon completion of the optimization phase.

One of the most frequently used optimizers nowadays, \emph{AdamW} \citep{loshchilov_decoupled_2019}, maintains both a first moment vector $\boldsymbol{\mu}_t$ and a second moment vector, $\boldsymbol{\nu}_t$, which are both exponential moving average functions of the gradients. In the optimizer, the latter serves as a diagonal preconditioner for the optimization step, adaptively rescaling the learning rate for each parameter. Various other optimizers also either maintain an estimate of second-order information or directly store this information in every iteration \citep{becker:improving,yao2021adahessian, liu2024sophia}. Similar approaches exist in variational learning. One example is the optimizer IVON \citep{shen_variational_2024}, which follows a similar idea by maintaining an online estimate of the diagonal of the (expected) Hessian to approximate the posterior variance for variational inference. Our proposed optimizer-informed preconditioning is agnostic to the specific training algorithm and applies to any method that maintains curvature-related statistics.

\subsection{Theoretical Connections: Fisher Information}

The idea of using geometry to guide sampling is theoretically well-established, with foundational work on Riemannian Manifold MCMC demonstrating significant efficiency gains by defining a metric tensor based on the Fisher Information Matrix \citep[FIM;][]{girolami_riemann_2011}. Thereby, the FIM is mostly considered the theoretically optimal choice of preconditioner as it fully captures the local geometry of the likelihood \citep{ma_complete_2015}. However, computing the full FIM is intractable for modern architectures.

A critical insight is that the diagonal of the FIM can be cheaply approximated. Recent work has established a strong connection between the statistics kept by adaptive optimizers and the diagonal FIM, a concept sometimes termed ``Fisher for free'' \citep{li_fishers_2025}. Specifically, the squared gradient accumulator from optimizers like Adam serves as a computationally trivial approximation to the empirical FIM diagonal. While this connection is known, it has not been systematically studied as a mechanism to bridge the optimization and sampling stages in the BDE context. Most existing preconditioning methods for SGMCMC either rely on costly online adaptation during warmup \citep{springenberg_bayesian_2016} or require periodic re-estimation \citep{wenzel_how_2020}. Our work investigates a simpler, more direct approach: leveraging the final state of the optimizer's curvature estimate as a static, high-quality preconditioner for the entire subsequent sampling run, thereby aiming to reduce or eliminate the adaptive warmup phase entirely.

\subsection{Applicability to Non-Gradient Samplers}
While our primary evaluation focuses on the kinetic dynamics of SGHMC, the principle of optimizer-derived geometric transfer extends naturally to non-gradient-based samplers, such as Gibbs sampling. In high-dimensional Bayesian neural networks, blocked Gibbs samplers often suffer from vanishing acceptance rates due to isotropic proposal distributions that fail to account for parameter correlations \citep{papamarkou_2023_ApproximateBlocked}. By repurposing the second-moment proxies like $\boldsymbol{\nu}_T$ as the scales of an anisotropic proposal density, one could transform the standard spherical proposals into ellipsoids aligned with the local posterior curvature. This geometric alignment would allow the sampler to take larger, more effective steps without the associated penalty in rejection rates, essentially solving the ``vanishing acceptance" problem through an inexpensive geometric warmstart rather than complex block-size reduction strategies.

\section{PROOFS AND METHOD DETAILS} \label{app:proofs}

\subsection{Proof of Proposition 3.1} \label{app:proof-3.1}

\begin{proof}
The SGHMC dynamics are governed by the Hamiltonian $H(\params, \mathbf{r}) = U(\params) + \frac{1}{2}\mathbf{r}^\top \mathbf{M}^{-1} \mathbf{r}$. To analyze the efficiency of the sampling trajectory, we consider the geometry of the potential within the local basin $\mathcal{B}(\theta^*)$ \citep{paulin_sampling_2025}. We ignore stochastic gradient noise and friction terms, which do not affect the local spectral argument given in the following. Under the local Gaussian approximation, the potential energy is $U(\params) \approx \frac{1}{2} \|\params - \params^*\|^2_{\mathbf{H}^*}$. Using the linear coordinate transformation $\tilde{\params} = \mathbf{M}^{1/2}(\params - \params^*)$, the potential in the transformed space becomes:
\begin{align*}
        \tilde{U}(\tilde{\params}) &= \frac{1}{2} (\mathbf{M}^{-1/2} \tilde{\params})^\top \mathbf{H}^* (\mathbf{M}^{-1/2} \tilde{\params})\\ &= \frac{1}{2} \tilde{\params}^\top (\mathbf{M}^{-1/2} \mathbf{H}^* \mathbf{M}^{-1/2}) \tilde{\params}.
\end{align*}
Substituting the initialization $\mathbf{M} = \mathbf{H}^*$ into the transformed Hessian $\tilde{\mathbf{H}} = \mathbf{M}^{-1/2} \mathbf{H}^* \mathbf{M}^{-1/2}$ yields
    $\tilde{\mathbf{H}} = (\mathbf{H}^*)^{-1/2} \mathbf{H}^* (\mathbf{H}^*)^{-1/2} = \mathbf{I}$.
The transformed potential is thus $\tilde{U}(\tilde{\params}) = \frac{1}{2} \|\tilde{\params}\|^2$, which is isotropic. Consequently, the condition number $\kappa(\tilde{\mathbf{H}}) = \lambda_{\max}/\lambda_{\min} = 1$, minimizing the localized mixing time and removing anisotropy as a source of poor local mixing.
\end{proof}

\subsection{Details on Layer-wise Normalization} \label{app:methods_details}

Inspired by work suggesting that parameters within a neural network layer share statistical properties \citep{sommer_connecting_2024, adilova2024layerwise}, we evaluate a layer-wise preconditioning strategy. This approach, following \citet{wenzel_how_2020}, aggregates the parameter-wise statistics to assign a single preconditioning value to all parameters within a given layer or parameter group (e.g. kernels and biases).

Let $\mathbf{m} = (m_1, \dots, m_d)$ be the vector of parameter-wise preconditioning values derived from one of the methods above. For each of the $L$ layers, we compute a single scalar value by averaging the values of its constituent parameters

\begin{equation}
m^{(\ell)} = \frac{1}{|\mathcal{I}_\ell|} \sum_{i \in \mathcal{I}_\ell} m_i,
\end{equation}
where $\mathcal{I}_\ell$ is the set of indices for parameters in layer $\ell$. To make the overall scale of the preconditioner independent of the step size $\eta$, we normalize these values:
%
    $\bar{m}^{(\ell)} = m^{(\ell)}/({\min_{\ell'=1,\dots,L} m^{(\ell')}})$.
%
The final layer-wise mass matrix $\mathbf{M}_{\text{layer}}$ is then constructed by assigning the normalized scalar $\bar{m}^{(\ell)}$ to all parameters in layer $\ell$. As teased in above, this reduces the granularity of the preconditioner but may offer a more stable and robust alternative, since the normalization naturally upper bounds the maximum parameterwise step size to $\eta$, while allowing smaller updates if necessary.

\section{EXPERIMENTAL SETUP \& FURTHER DETAILS} \label{app:experimental-setup}

\paragraph{Software and Computing Environment}
Our implementation is built in Python using \texttt{jax} \citep{jax2018github} and \texttt{BlackJAX} \citep{cabezas2024blackjax} and builds on the codebase of \citet{kobialka2026interplay}. We ran our experiments on NVIDIA A100 and NVIDIA RTX A6000 GPUs. A reproducible codebase is available at \url{https://github.com/EmanuelSommer/geopost_optimal2026}.

\paragraph{Computational Cost}
Runtime scales with model size and dataset complexity. For the LeNet architecture on F-MNIST, warmstarting across 4 parallel chains requires 1--1.5 hours, while the subsequent sampling phase takes up to 1 hour in the most demanding configuration (50,000 warmup steps, 100,000 samples). Our experimental grid for F-MNIST consists of 270 runs (3 seeds $\times$ 3 warmup lengths $\times$ 3 step sizes $\times$ 5 preconditioning strategies $\times$ 2 granularities), totaling approximately 500 GPU hours for this dataset alone. CIFAR-10 experiments exhibited similar computational demands. In contrast, the smaller FCN models and the nanoGPT architecture required only approximately 1/4 of the time per run. Across all experiments, the total computational cost amounts to approximately 1,750 GPU hours on NVIDIA A100 and RTX A6000 hardware.

\paragraph{Datasets}
Our empirical evaluation is conducted on a diverse suite of benchmark datasets, detailed in \cref{tab:dataoverview}. For tabular regression tasks, we partitioned the data into 70\% for training, 10\% for validation, and 20\% for testing. For the image classification benchmarks, we utilized the official train and test sets. Before training the nanoGPT model, we translated the \textit{tinyshakespeare} dataset available from \cite{Karpathy2022} into more modern English using Gemini 2.5 Pro to facilitate a more accessible assessment of the quality of the generated text. This was done using the prompt ``\,\texttt{Please translate the attached file into simplified modern English. Keep the structure of the text as is (new line for each speaker, speaker name followed by a colon, then the sentence in a new line)}'', together with an attached .txt file of the original text.  We call this dataset \textit{modern-shakespeare} and provide it within our public code repository for reproducibility. This adaptation allows for a more intuitive assessment of the generated text's quality.

\begin{table}[h]
\begin{small}
\begin{center}
\caption{Overview of the used datasets.} \label{tab:dataoverview}
\resizebox{0.6\columnwidth}{!}{%
\begin{tabular}{lrrl}
\textbf{Dataset}  & \textbf{Size} & \textbf{Features} & \textbf{Reference} \\ \hline 
bikesharing & 17379 & 13 & \citet{misc_bike_sharing_dataset_275}  \\
protein & 45730 & 9 & \citet{Dua.2019} \\
F(ashion)-MNIST & 60000 & 28x28 & \citet{xiao2017/online} \\
CIFAR-10 & 60000 & 3x32x32 & \citet{krizhevsky2009learning} \\
(modern-)shakespeare & 39890 & 65 & adapted from \citet{Karpathy2022}\\
\hline
\end{tabular}
}
\end{center}
\end{small}
\end{table}

\paragraph{Models} Since we are performing our experiments on different datasets and modalities, we also use different architectures throughout. In  \cref{tab:models} we provide a list of all different models together with their parameter count and the dataset we applied the model onto. The parameter count covers a relatively wide range to test the applicability of preconditioning on highly different model sizes.

\begin{table}[h]
\begin{small}
\begin{center}
\caption{Overview of the used models.} \label{tab:models}
\begin{tabular}{lrrl}
\textbf{Model}  & \textbf{\#Params} & \textbf{Dataset} \\ 
\hline 
FCN & 10k & bikesharing, protein \\
LeNet (CNN) &  60k & F-MNIST \\
ResNet-7 & 428k & CIFAR-10 \\
nanoGPT & 10.8M & shakespeare \\
\hline
\end{tabular}
\end{center}
\end{small}
\end{table}

For the tabular regression tasks on the bikesharing and protein datasets, we employ a Fully Connected Network (FCN) consisting of six hidden layers.

On the F-MNIST dataset, we use a classic LeNet architecture. It features two convolutional layers with 6 and 16 filters, followed by three fully-connected layers. Average pooling is applied after each convolutional block.

For image classification on CIFAR-10 we use a tailored ResNet-7 comprising 428k parameters. Instead of BatchNorm, we employ Filter Response Normalization (FRN) \citep{singh2020filter}, motivated by known limitations of BatchNorm in sampling-based settings \citep{wenzel_how_2020, shen_variational_2024}. The full specification of the architecture is provided in \cref{tab:arch-ResNet-7}.

\begin{table}[h]
\begin{small}
\begin{center}
\caption{The Custom ResNet-7 Architecture with 428k trainable parameters. The output shape is specified for a sample input tensor of size $3 \times 32 \times 32$. All convolutional layers use a $3 \times 3$ kernel, stride 1, and 'SAME' padding, and are followed by Filter Response Normalization \citep[FRN;][]{singh2020filter}.}
\label{tab:arch-ResNet-7}
\begin{tabular}{llc}
\textbf{Stage} & \textbf{Layer Operation(s)} & \textbf{Filters} \\
\hline
\texttt{input} & Image & - \\
\hline
\texttt{stem} & Conv-FRN & 32 \\
\hline
\texttt{body} & Conv-FRN $\rightarrow$ MaxPool & 64 \\
& Conv-FRN & 64 \\
& Conv-FRN $\rightarrow$ MaxPool & 128 \\
& Conv-FRN $\rightarrow$ MaxPool & 128 \\
\hline
\texttt{res\_block} & \textit{(Identity Shortcut from previous output)} & 128 \\
& \quad $\hookrightarrow$ Conv-FRN & 128 \\
& \quad $\hookrightarrow$ Add & 128 \\
\hline
\texttt{head} & Global Average Pool &  - \\
& Fully Connected &  - \\
\hline
\end{tabular}
\end{center}
\end{small}
\end{table}

The setup for character-level modeling on the shakespeare dataset, is as follows. We use a 6-layer, 6-head GPT-style transformer with context length 256 and embedding dimension 384, adapted from \citet{Karpathy2022}. Dropout is disabled. The learning rate during warmstarting decays linearly from $3\mathrm{e}{-4}$ to $2.5\mathrm{e}{-4}$.

\newpage

\paragraph{Sampling Configuration}
For each distinct experimental configuration, we conducted 3 runs with different random seeds. For SGHMC, we explored the different preconditioning strategies: identity, gradient-based, and 1st and 2nd moment-based both parameter- and layer-wise. We run parallel chains thinned by a factor of 100 using normal isotropic priors. The specific hyperparameter settings for each dataset are detailed in \cref{tab:experiment-grid}. AdaSGHMC was evaluated with the same configurations, but the zero warmup case was excluded, as it would be identical to identity-preconditioned SGHMC. In the main results tables, we report its performance with 5000 warmup steps for all datasets except for the shakespeare dataset where we use 500 warmup steps.

\begin{table}[h!]
\begin{small}
\begin{center}
\centering
\caption{Sampler configurations for each dataset. The best performing step size per dataset, used in the main results tables, are boldfaced.}
\label{tab:experiment-grid}
\begin{tabular}{llllc}
\textbf{Dataset} & \textbf{Batch Size} & \textbf{Warmup Steps} & \textbf{Step Sizes} & \textbf{Chains} \\
\hline
bikesharing  & 128 & {0, 5000, 50000} & {\textbf{0.0001}, 0.001, 0.01} & {8} \\
protein     & 128 &  {0, 5000, 50000} & {\textbf{0.0001}, 0.001, 0.01} & 8 \\
F-MNIST      & 256 & 0, 5000, 50000 & 0.0001, \textbf{0.001}, 0.01 & 10 \\
CIFAR-10    & 518 & 0, 5000, 50000 & \textbf{0.0001}, 0.001, 0.01 & 8 \\
shakespeare & 128 & 0, 500, 5000 & 0.0001, \textbf{0.0002}, 0.001 & 4 \\
\hline
\end{tabular}
\end{center}
\end{small}
\end{table}

\paragraph{Tuning Warmstarting Optimizers}\label{app:tuning-warmstarting-optimizers}
Prior to running the main experimental grid, we conducted a hyperparameter (HP) tuning study to determine which configuration of the warmstarting optimizers we will use throughout our experiments. This preliminary investigation prevents suboptimal warmstarting configurations from influencing the final results. We tuned two key parameters of AdamW with three discrete values each, yielding this grid:

\begin{itemize}
   \item \textbf{Weight Decay} ($\weightdecay$): \{0.001, 0.01, 0.02\}
   \item \textbf{Learning Rate} ($\alpha$): \{0.0001, 0.001, 0.01\}
\end{itemize}

The {HP} configurations were evaluated using both accuracy and {LPPD}, with the optimal configuration selected based on the best performance across both measures. In our experiments, configurations that maximized accuracy also yielded the highest LPPD value, therefore leaving the choice of the best combination straightforward.

The results across the $3\times3$ {HP} grid are shown in \cref{fig:hp_tuning_adamw}. The contour plots display performance surfaces created through linear interpolation between grid points. The marginal plots on each axis show the averaged performance along each parameter dimension. The color gradients indicate performance levels, with warmer colors representing lower performance and cooler colors indicating superior performance. The black dots mark the grid points where we evaluated the performance.

The HP tuning study yielded the following optimal configurations, which are used throughout all experiments in this work: $\beta_1 = 0.9$, $\beta_2 = 0.999$, $\alpha = 0.01$, $\weightdecay = 0.02$.

\begin{figure}[h]
   \centering
   \begin{subfigure}{0.45\textwidth}
       \centering
       \includegraphics[height=5cm]{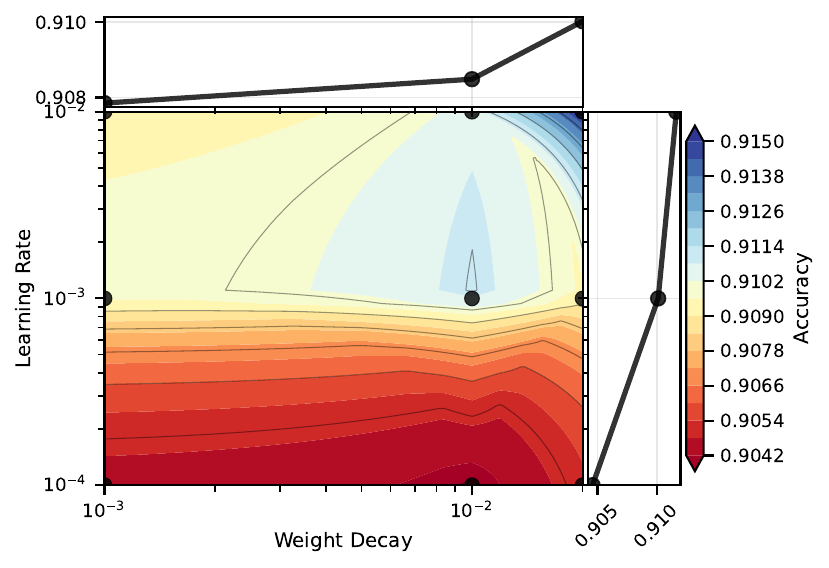}
       \caption{Accuracy}
   \end{subfigure}
   \hfill
   \begin{subfigure}{0.45\textwidth}
       \centering
       \includegraphics[height=5cm]{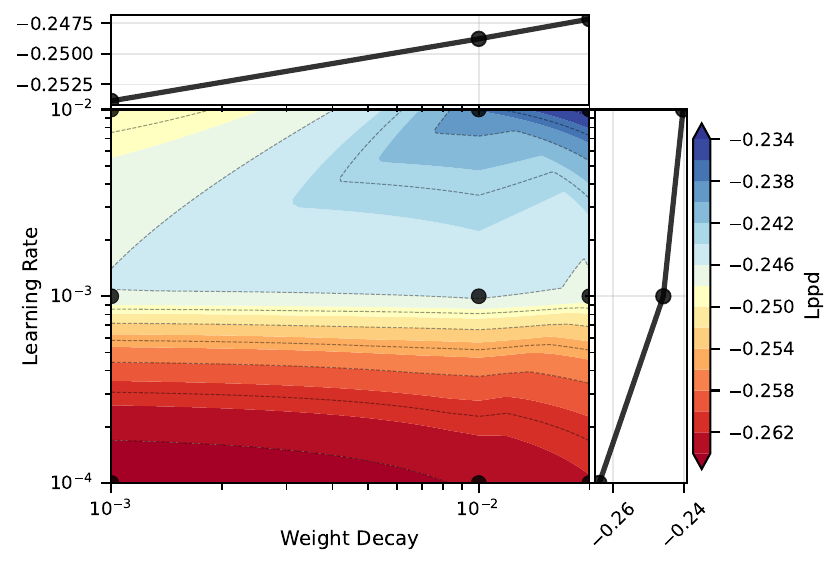}
       \caption{Lppd}
   \end{subfigure}
   \caption[Performance after warmstarting with AdamW]{Impact of the HP choice on the warmstarting performance of AdamW for F-MNIST.}
   \label{fig:hp_tuning_adamw}
\end{figure}

\newpage
\section{ADDITIONAL EMPIRICAL ANALYSIS} \label{app:further-results}

\subsection{Warmup and Runtime Analysis}
\label{app:warmup_analysis}

Among the proposed methods, gradient- and first-moment-based preconditioning emerge as the best choice when performing zero warmup, particularly in terms of log pointwise predictive density (LPPD). As illustrated in Figures \ref{fig:warmup-effects-fmnist} and \ref{fig:warmup-effects-protein}, these local, geometry-informed strategies achieve peak performance immediately and do not benefit from an extended warmup period. The performance only improves minimal despite the sampler given a much higher sampling budget. In stark contrast, the performance of second-moment preconditioning improves substantially with a longer warmup. We hypothesize this is because the second-moment statistic captures more global, richer geometric information about the posterior landscape, which requires a longer sampling phase to be fully exploited.

\begin{figure}[h!]
    \centering
    \includegraphics[width=0.45\linewidth]{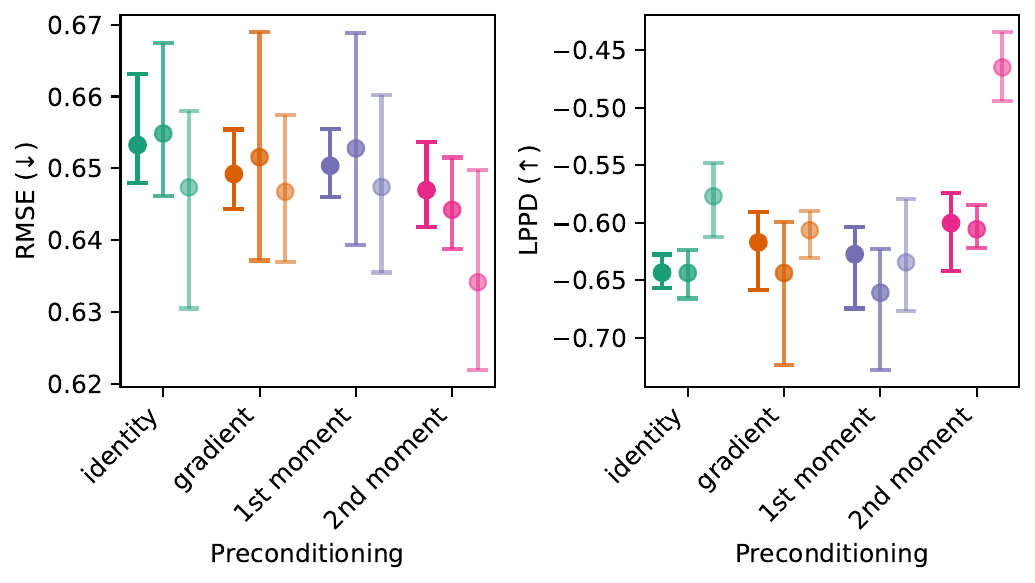}
    \caption{Minimal, maximal, and average performance on the protein dataset for different warmup lengths after warmstarting with AdamW and sampling with SGHMC using the best step size per dataset.}
    \label{fig:warmup-effects-protein}
\end{figure}

This finding highlights a trade-off in computational budget choice, which influences the preferred preconditioning method. Extended sampling runs (see \cref{fig:running-fmnist}) confirm that these performance differences are stable, with each method converging to a distinct performance level based on the chosen warmup. Consequently, gradient and first-moment methods are preferable for budget-constrained scenarios, while second-moment preconditioning may offer superior performance if a longer warmup phase can be afforded.

\subsection{Numerical Robustness}
\label{app:rob}

As shown in \cref{fig:robustness_stepsize_resnet}, vanilla SGHMC performance degrades sharply with misspecified step sizes, while preconditioned variants maintain stable performance across a significantly wider range. This robustness stems from the distinction between the user-defined global $\eta$ and the parameter-specific effective step size $(\eta \cdot m_i^{-1})$. Specifically, the layer-wise normalization (\cref{sec:methods}) acts as a numerical safeguard: even with a large global $\eta$ intended to accelerate convergence, effective updates are rescaled to remain within a regime where discretization bias is controlled.

\begin{figure}[h!]
    \centering
    \includegraphics[width=0.5\linewidth]{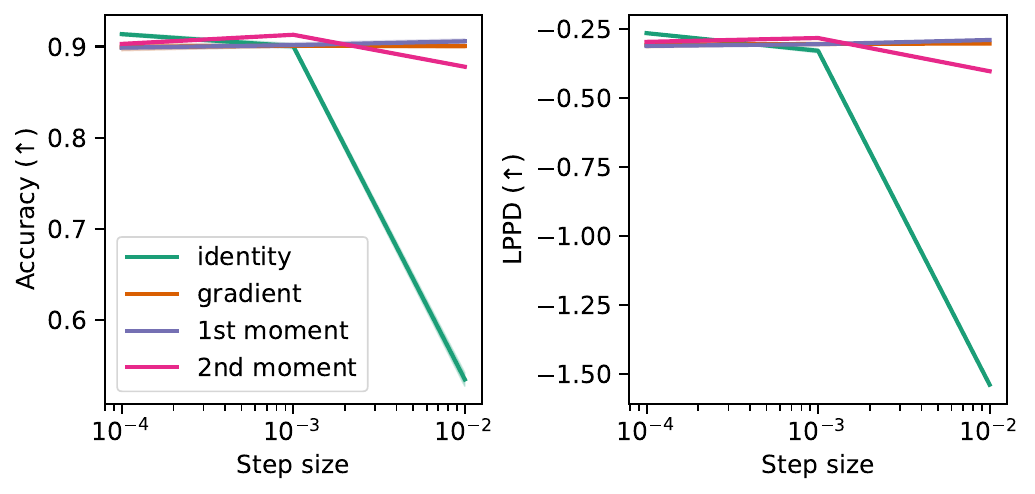}
    \caption{Robustness on CIFAR-10: Preconditioned variants maintain high Accuracy (left) and LPPD (right) across a wider range of step sizes compared to vanilla SGHMC.}
    \label{fig:robustness_stepsize_resnet}
    \vspace{-1em} 
\end{figure}

This mitigates the risk of performance collapse from suboptimal hyperparameters, aiding practical usability. While consistent across datasets, we note that for excessively large global step sizes, effective updates eventually exceed local basin bounds, leading to non-convergence for all methods.

\subsection{Influence on Runtime and Sampling Budgets}\label{app:runtime}

A key advantage of our optimizer-informed preconditioning is the potential for substantial reductions in total sampling budget. As shown in \cref{fig:running-fmnist}, samplers achieve peak performance very early in the sampling run, with accuracy curves quickly flattening. This effect is pronounced even with zero warmup (left plot), making lengthy warmup phases unnecessary and allowing for complete elimination of the warmup period, as demonstrated by our main results in \cref{tab:results_combined}.

\begin{figure*}[h!]
    \centering
    \includegraphics[width=.92\textwidth]{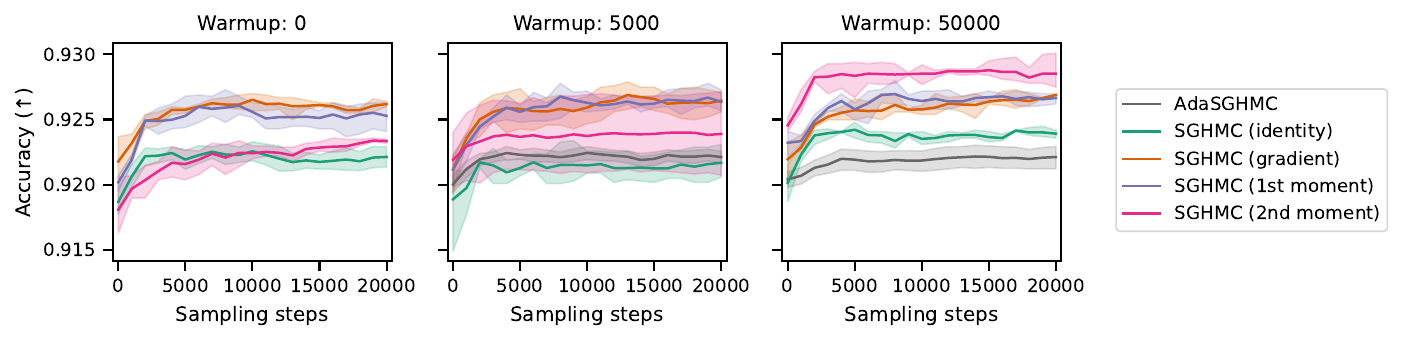}
    \caption{Running mean of the predictive performance with std.\ dev.\ (shaded) during the sampling phase after warmstarting with AdamW and for different warmup lengths on F-MNIST.}
    \label{fig:running-fmnist}
\end{figure*}

This efficiency is most pronounced for local preconditioning strategies (gradient and first-moment), which achieve nearly identical performance trajectories irrespective of warmup duration, making them ideal for rapid, low-budget inference. For a pipeline with 5,000 warmup and 10,000 sampling steps—typical for BDE settings—our approach can eliminate the warmup phase, reducing total computational cost by 33\% while enabling performance gains. Conversely, if larger computational budgets permit long warmup, the advantage shifts towards second-moment preconditioning, which can leverage the more globally attained geometry estimate to ultimately outperform all other configurations.

\subsection{Empirical Preconditioning Distributions}
\label{app:empprecond}
We also investigate the distributions of the preconditioning values in the $\mathbf {M}$ matrix over all chains to understand how they adapt to the network structure. \cref{fig:distr-layerw-precon} reveals consistent layer-wise patterns across all methods for the LeNet architecture: the first and last layers systematically receive higher preconditioning values, which reduces the effective step size for their parameters and is in line with recent findings on layer-wise exploration of high-performance full-batch samplers \citep{sommer_connecting_2024}. These distinct, non-uniform patterns, which are particularly pronounced for second-moment preconditioning, lend strong empirical support to the layer-wise simplification. Furthermore, we observe that gradient- and first-moment preconditioning distributions are nearly identical. This explains the similar performance metrics reported for these two methods when comparing the relative improvements over the different datasets.

\begin{figure}[h!]
    \centering
    \includegraphics[width=0.45\linewidth]{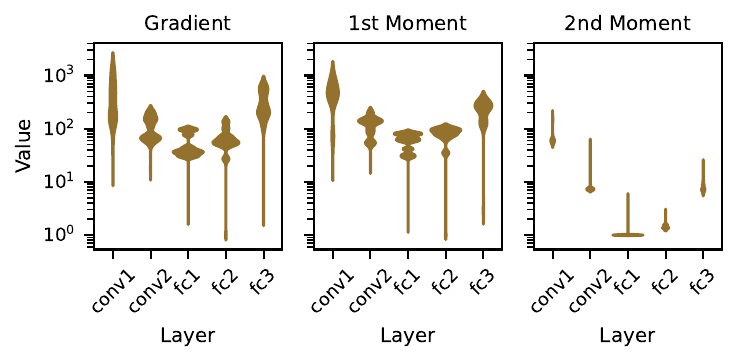}
    \caption{Layer-wise distributions of parameter values in the preconditioning matrix after warmstarting with AdamW on LeNet.}
    \label{fig:distr-layerw-precon}
\end{figure}

\subsection{Adaptive Preconditioning Variants}
\label{app:adaptive_analysis}

Interestingly, we find that the positive influence of our geometry transfer above does not translate to SGHMC variants with scale-adaptations such as AdaSGHMC. To understand why, we analyze the evolution of AdaSGHMC's internal mass matrix $\mathbf M$. AdaSGHMC dynamically updates $\mathbf M$ during warmup to learn the local geometry. Our analysis reveals that this internal adaptation effectively \enquote{overwrites} any initial preconditioning. \cref{fig:distr-layerw-precon-adasghmc} substantiates this claim, displaying the distribution of values in $\mathbf M$ after adaptation. The left subplot shows the result when starting from identity preconditioning, while the right subplot starts from our first-moment-based preconditioning strategy. Visually, the two distributions are nearly indistinguishable, showing that marginalized over the layers, the sampler results in the same geometric configuration regardless of the starting point.

We further quantify this phenomenon using pairwise cosine similarity analysis across ensemble chains (see \cref{fig:similarity-preconditioning}), uncovering two key insights. First, the analysis validates our overwriting hypothesis: after adaptation, mass matrices initialized from identity and from the first moment are highly similar, with nearly identical third and fourth boxes in the left and right subplot each. Second, AdaSGHMC's adaptation actively reduces ensemble diversity. The adapted matrices exhibit extremely high self-similarity (near 1 with almost no variance in the right subplot), meaning mass matrices become almost identical across all chains. In contrast, the static first-moment preconditioner (second box of the left subplot) shows lower self-similarity values. We hypothesize this preserves unique geometric information from each chain's starting mode. This loss of diversity enforced by scale-adaptation likely explains why our simpler, diversity-preserving preconditioned SGHMC often achieves superior generalization performance.

\begin{figure}[h!]
    \centering
    \includegraphics[width=0.55\linewidth]{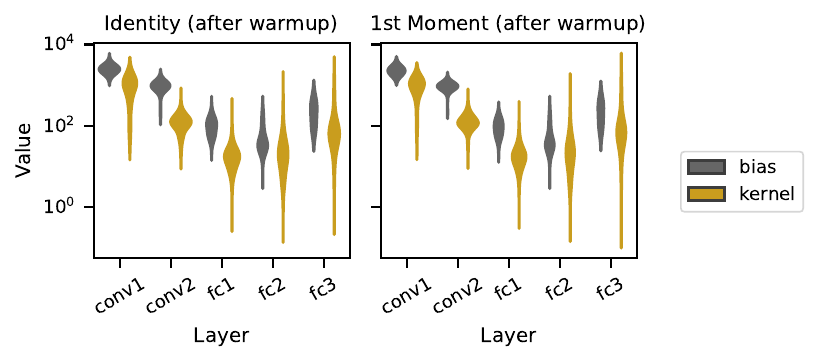}
    \caption{Layer-wise distributions of bias and kernel parameters on the diagonal of the corresponding mass matrices after warmup adaptation throughout the layers. We show the values of $\mathbf M$ after identity- (left) and after first moment-based preconditioning (right).}
    \label{fig:distr-layerw-precon-adasghmc}
\end{figure}

\begin{figure}[h!]
    \centering
    \includegraphics[width=0.5\linewidth]{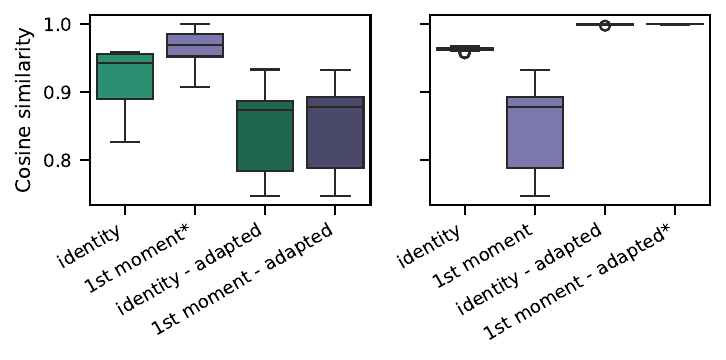}
    \caption{Cosine similarity of $\mathbf M$ matrices after pairwise comparison with the first moment-based preconditioning matrix (left) and the adapted first moment-based preconditioning matrix (right). Self-similarities are marked by a star (*).}
    \label{fig:similarity-preconditioning}
\end{figure}

\subsection{Performance Overview of Layer-wise Preconditioning}\label{app:furtherlayerwiseperf}
Since we provide only average performance scores and the improvement over the DE baseline \cref{tab:results_combined}, we detail our results here, including also the standard error over three runs. Tables \ref{tab:results_predperf} and \ref{tab:results_lppd} present these details for predictive performance and uncertainty quantification, respectively. The results further confirm the consistent improvements achieved by our proposed preconditioning strategies, as discussed above.

\begin{table*}[h!]
    \centering
    \small
    \caption{Predictive performance of layer-wise preconditioning across datasets and inference methods. Reported are RMSE ($\downarrow$) for regression tasks and Accuracy ($\uparrow$) for classification tasks. Values denote mean $\pm$ standard error over three independent runs.}
    \label{tab:results_predperf}
    \setlength{\tabcolsep}{4pt}
    \resizebox{\textwidth}{!}{
    \begin{tabular}{ll
        c c
        c c c}
    \toprule
    \multicolumn{1}{c}{\multirow{2}{*}{\textbf{Sampler}}} &
    \multicolumn{1}{c}{\multirow{2}{*}{\textbf{Preconditioning}}} &
    \multicolumn{2}{c}{\textbf{RMSE} ($\downarrow$)} &
    \multicolumn{3}{c}{\textbf{Accuracy} ($\uparrow$)} \\
    \cmidrule(lr){3-4}\cmidrule(lr){5-7}
      & & \textbf{bike} & \textbf{protein} &
      \textbf{CIFAR-10} & \textbf{F-MNIST} & \textbf{shakespeare} \\
    \midrule
    Deep Ensemble & -- & 0.2445 ± 0.0089 & 0.6589 ± 0.0015 & 0.9019 ± 0.0001 & 0.9157 ± 0.0015 & 0.5504 ± 0.0030 \\
    \midrule
    SGHMC & identity & 0.2392 ± 0.0090 & 0.6533 ± 0.0049 & 0.9150 ± 0.0014 & 0.9226 ± 0.0004 & 0.5536 ± 0.0022 \\
    SGHMC & gradient & 0.2406 ± 0.0089 & 0.6492 ± 0.0033 & 0.9155 ± 0.0007 & 0.9268 ± 0.0002 & 0.5533 ± 0.0029 \\
    SGHMC & 1st moment & 0.2405 ± 0.0087 & 0.6504 ± 0.0028 & 0.9141 ± 0.0004 & 0.9256 ± 0.0011 & 0.5543 ± 0.0022 \\
    SGHMC & 2nd moment & 0.2400 ± 0.0089 & 0.6470 ± 0.0035 & 0.9154 ± 0.0012 & 0.9231 ± 0.0004 & 0.5535 ± 0.0016 \\
    \midrule
    AdaSGHMC & adaptive & 0.2424 ± 0.0087 & 0.6560 ± 0.0023 & 0.9055 ± 0.0025 & 0.9197 ± 0.0005 & 0.5523 ± 0.0027 \\
    \bottomrule
    \end{tabular}
    }
\end{table*}

\begin{table*}[h!]
    \centering
    \small
    \caption{Uncertainty quantification performance of layer-wise preconditioning across datasets and inference methods. Reported are LPPD ($\uparrow$) and perplexity ($\downarrow$). Values denote mean ± standard error over three independent runs.}
    \label{tab:results_lppd}
    \setlength{\tabcolsep}{4pt}
    \resizebox{\textwidth}{!}{
    \begin{tabular}{ll
        c c c c
        c}
    \toprule
    \multicolumn{1}{c}{\multirow{2}{*}{\textbf{Sampler}}} &
    \multicolumn{1}{c}{\multirow{2}{*}{\textbf{Preconditioning}}} &
    \multicolumn{4}{c}{\textbf{LPPD} ($\uparrow$)} &
    \multicolumn{1}{c}{\textbf{Perplexity} ($\downarrow$)} \\
    \cmidrule(lr){3-6}\cmidrule(lr){7-7}
      & & \textbf{bike} & \textbf{protein} &
      \textbf{CIFAR-10} & \textbf{F-MNIST} & \textbf{shakespeare} \\
    \midrule
    Deep Ensemble & -- & 0.5834 ± 0.0206 & -0.7173 ± 0.0143 & -0.3047 ± 0.0025 & -0.2318 ± 0.0002 & 4.3687 ± 0.0310 \\
    \midrule
    SGHMC & identity & 0.6190 ± 0.0163 & -0.6434 ± 0.0086 & -0.2641 ± 0.0025 & -0.2216 ± 0.0002 & 4.2645 ± 0.0250 \\
    SGHMC & gradient & 0.6418 ± 0.0179 & -0.6171 ± 0.0210 & -0.2583 ± 0.0012 & -0.2102 ± 0.0008 & 4.2846 ± 0.0272 \\
    SGHMC & 1st moment & 0.6416 ± 0.0172 & -0.6274 ± 0.0236 & -0.2644 ± 0.0012 & -0.2118 ± 0.0001 & 4.2801 ± 0.0287 \\
    SGHMC & 2nd moment & 0.6397 ± 0.0179 & -0.6002 ± 0.0210 & -0.2615 ± 0.0020 & -0.2214 ± 0.0002 & 4.2894 ± 0.0412 \\
    \midrule
    AdaSGHMC & adaptive & 0.6182 ± 0.0192 & -0.6908 ± 0.0173 & -0.2942 ± 0.0011 & -0.2197 ± 0.0016 & 4.3382 ± 0.0382 \\
    \bottomrule
    \end{tabular}
    }
\end{table*}

\subsection{Performance Overview of Parameter-wise Preconditioning}\label{app:furtherparamwiseperf}
Tables \ref{tab:results_predperf_paramwise} and \ref{tab:results_lppd_paramwise} detail the results when applying parameter- instead of layer-wise preconditioning. The missing entries highlight the numerical instabilities encountered with this more granular strategy on several datasets. This empirically validates our choice to use the layer-wise approach for the main experiments, as it proved to be more robust and did not suffer from such failures in any of our tested configurations.

\begin{table*}[h!]
    \centering
    \small
    \caption{Predictive performance of parameter-wise preconditioning across datasets. Reported are the RMSE ($\downarrow$) for regression tasks and Accuracy ($\uparrow$) for classification tasks. Values denote mean ± standard error over three independent runs.}
    \label{tab:results_predperf_paramwise}
    \setlength{\tabcolsep}{4pt}
    \resizebox{\textwidth}{!}{
    \begin{tabular}{ll
        c c
        c c c}
    \toprule
    \multicolumn{1}{c}{\multirow{2}{*}{\textbf{Sampler}}} &
    \multicolumn{1}{c}{\multirow{2}{*}{\textbf{Preconditioning}}} &
    \multicolumn{2}{c}{\textbf{RMSE} ($\downarrow$)} &
    \multicolumn{3}{c}{\textbf{Accuracy} ($\uparrow$)} \\
    \cmidrule(lr){3-4}\cmidrule(lr){5-7}
      & & \textbf{bike} & \textbf{protein} &
      \textbf{CIFAR-10} & \textbf{F-MNIST} & \textbf{shakespeare} \\
    \midrule
    Deep Ensemble & -- & 0.2445 ± 0.0089 & 0.6589 ± 0.0015 & 0.9019 ± 0.0001 & 0.9157 ± 0.0015 & 0.5504 ± 0.0030 \\
    \midrule
    SGHMC & identity & 0.2392 ± 0.0090 & 0.6533 ± 0.0049 & 0.9150 ± 0.0014 & 0.9226 ± 0.0004 & 0.5536 ± 0.0022 \\
    SGHMC & gradient & -- & -- & 0.9128 ± 0.0020 & -- & 0.5535 ± 0.0020 \\
    SGHMC & 1st moment & -- & -- & 0.9153 ± 0.0006 & -- & 0.5524 ± 0.0017 \\
    SGHMC & 2nd moment & 0.2426 ± 0.0090 & 0.6510 ± 0.0037 & 0.9150 ± 0.0009 & 0.9228 ± 0.0012 & 0.5539 ± 0.0015 \\
    \bottomrule
    \end{tabular}
    }
\end{table*}

\begin{table*}[h!]
    \centering
    \small
    \caption{Uncertainty quantification performance of parameter-wise preconditioning across datasets. Reported are LPPD ($\uparrow$) and perplexity ($\downarrow$). Values denote mean ± standard error over three independent runs.}
    \label{tab:results_lppd_paramwise}
    \setlength{\tabcolsep}{4pt}
    \resizebox{\textwidth}{!}{
    \begin{tabular}{ll
        c c c c
        c}
    \toprule
    \multicolumn{1}{c}{\multirow{2}{*}{\textbf{Sampler}}} &
    \multicolumn{1}{c}{\multirow{2}{*}{\textbf{Preconditioning}}} &
    \multicolumn{4}{c}{\textbf{LPPD} ($\uparrow$)} &
    \multicolumn{1}{c}{\textbf{Perplexity} ($\downarrow$)} \\
    \cmidrule(lr){3-6}\cmidrule(lr){7-7}
      & & \textbf{bike} & \textbf{protein} &
      \textbf{CIFAR-10} & \textbf{F-MNIST} & \textbf{shakespeare} \\
    \midrule
     Deep Ensemble & -- & 0.5834 ± 0.0206 & -0.7173 ± 0.0143 & -0.3047 ± 0.0025 & -0.2318 ± 0.0002 & 4.3687 ± 0.0310 \\
    \midrule
    SGHMC & identity & 0.6190 ± 0.0163 & -0.6434 ± 0.0086 & -0.2641 ± 0.0025 & -0.2216 ± 0.0002 & 4.2645 ± 0.0250 \\
    SGHMC & gradient & -- & -- & -0.2678 ± 0.0057 & -- & 4.2770 ± 0.0294 \\
    SGHMC & 1st moment & -- & -- & -0.2635 ± 0.0017 & -- & 4.2871 ± 0.0238 \\
    SGHMC & 2nd moment & 0.5646 ± 0.0104 & -0.6223 ± 0.0192 & -0.2635 ± 0.0026 & -0.2213 ± 0.0007 & 4.2964 ± 0.0211 \\
    \bottomrule
    \end{tabular}
    }
\end{table*}

\newpage
\section{OPTIMIZATION DETAILS}\label{app:optimizers}

We provide pseudo-code for the AdamW optimizer (\cref{alg:adamw}) using notation aligned with the formalization of the preconditioning schemes. Presenting the optimizer in this form makes explicit that the first- and second-moment statistics required by our optimizer-informed preconditioning are readily available as part of the standard optimization state. In particular, the terminal values of the moment accumulators $\boldsymbol{\mu}_t$ and $\boldsymbol{\nu}_t$ can be extracted at no additional computational cost and directly reused to initialize the sampler’s mass matrix.

\begin{algorithm}[h!]
\caption{AdamW~\citep{loshchilov_decoupled_2019}}
\begin{algorithmic}[1]
\Require Learning rate $\alpha > 0$, Momentum parameters $\beta_1, \beta_2 \in [0, 1)$, Weight decay $\weightdecay \geq 0$, Stability constant $\epsilon > 0$
\State \textbf{Init:} $\params_0 \leftarrow$ initial parameters, $\boldsymbol{\mu}_0 \leftarrow \mathbf{0}$ (first moment), $\boldsymbol{\nu}_0 \leftarrow \mathbf{0}$ (second moment)
\For{$t = 1, 2, \ldots$}
    \State $\mathbf{g}_t \leftarrow \nabla \ell(\params_{t-1})$
    \State $\boldsymbol{\mu}_t \leftarrow \beta_1 \boldsymbol{\mu}_{t-1} + (1-\beta_1)\mathbf{g}_t$
    \State $\boldsymbol{\nu}_t \leftarrow \beta_2 \boldsymbol{\nu}_{t-1} + (1-\beta_2)\mathbf{g}_t^2$
    \State ${\boldsymbol{\mu}}^\prime_t \leftarrow \boldsymbol{\mu}_t/(1-\beta_1^t)$
    \State ${\boldsymbol{\nu}}^\prime_t \leftarrow \boldsymbol{\nu}_t/(1-\beta_2^t)$
    \State $\params_t \leftarrow \params_{t-1} - \alpha({\boldsymbol{\mu}}^\prime_t+\lambda\params_{t-1})/(\sqrt{{\boldsymbol{\nu}}^\prime_t} + \epsilon)$
\EndFor
\State \Return $\params_t$
\end{algorithmic}
\label{alg:adamw}
\end{algorithm}

\end{document}